\documentclass[9pt,twocolumn]{article}
\usepackage[english]{babel}
\usepackage{amsmath,amssymb,amsthm,mathtools}
\usepackage{graphicx}
\usepackage{url}
\usepackage{pbox}
\usepackage{subfigure}
\usepackage{capt-of}
\usepackage{sidecap}
\usepackage{color}

\usepackage{times}
\usepackage{hyperref}

\newtheorem{theorem}{Theorem}[section]

\newtheorem{lemma}[theorem]{Lemma}
\newtheorem{claim}[theorem]{Claim}

\newtheorem{definition}[theorem]{Definition}

\newcommand{\ignore}[1]{}
\sidecaptionvpos{figure}{c}
\newcommand{\degr}{D}

\newcommand{\cG}{\mathcal{G}}

\newcommand{\cL}{{\cal L}}

\newcommand{\cN}{{\cal N}}
\newcommand{\cP}{\mathcal{P}}

\newcommand{\poly}{\mathrm{poly}}
\newcommand{\rank}{\mathrm{rank}}

\newcommand{\RR}{\mathbb{R}}

\newcommand{\EX}{\hbox{\bf E}}

\renewcommand{\epsilon}{\varepsilon}
\newcommand{\Sec}[1]{\hyperref[sec:#1]{\S\ref*{sec:#1}}} 
\newcommand{\Eqn}[1]{\hyperref[eq:#1]{Eqn.\,(\ref*{eq:#1})}} 
\newcommand{\Fig}[1]{\hyperref[fig:#1]{Fig.\,\ref*{fig:#1}}} 
\newcommand{\Tab}[1]{\hyperref[tab:#1]{Tab.\,\ref*{tab:#1}}} 
\newcommand{\Thm}[1]{\hyperref[thm:#1]{Theorem\,\ref*{thm:#1}}} 
\newcommand{\Fact}[1]{\hyperref[fact:#1]{Fact\,\ref*{fact:#1}}} 
\newcommand{\Lem}[1]{\hyperref[lem:#1]{Lemma\,\ref*{lem:#1}}} 
\newcommand{\Prop}[1]{\hyperref[prop:#1]{Prop.~\ref*{prop:#1}}} 
\newcommand{\Cor}[1]{\hyperref[cor:#1]{Corollary~\ref*{cor:#1}}} 
\newcommand{\Conj}[1]{\hyperref[conj:#1]{Conjecture~\ref*{conj:#1}}} 
\newcommand{\Def}[1]{\hyperref[def:#1]{Definition~\ref*{def:#1}}} 
\newcommand{\Alg}[1]{\hyperref[alg:#1]{Alg.~\ref*{alg:#1}}} 
\newcommand{\Ex}[1]{\hyperref[ex:#1]{Ex.~\ref*{ex:#1}}} 
\newcommand{\Clm}[1]{\hyperref[clm:#1]{Claim~\ref*{clm:#1}}} 
\newcommand{\Step}[1]{\hyperref[step:#1]{Step~\ref*{step:#1}}} 

\newcommand{\new}[1]{#1}
\newcommand{\si}[1]{}
\newcommand{\main}[1]{#1}

\title{The impossibility of low rank representations for triangle-rich complex networks\footnote{Published in the Proceedings of National Academy of Sciences, Mar 2020~\cite{SeSh+20}.}}

\author{C. Seshadhri\footnote{C. Seshadhri acknowledges the support of NSF Awards CCF-1740850, CCF-1813165, and ARO Award W911NF1910294.} \\
University of California, Santa Cruz\\
{\tt sesh@ucsc.edu}
\and Aneesh Sharma\\
Google\\
{\tt aneesh.x.sharma@gmail.com}
\and Andrew Stolman\\
University of California, Santa Cruz\\
{\tt astolman@ucsc.edu}\\
\and Ashish Goel\\
Stanford University\\
{\tt ashishg@stanford.edu}
}

\begin{document}

\date{}
\maketitle
\begin{abstract}
  The study of complex networks is a significant development in modern
  science, and has enriched the social sciences, biology, physics, and
  computer science. Models and algorithms for such networks are
  pervasive in our society, and impact human behavior via social
  networks, search engines, and recommender systems to name a few. A
  widely used algorithmic technique for modeling such complex networks
  is to construct a low-dimensional Euclidean embedding of the
  vertices of the network, where proximity of vertices is interpreted
  as the likelihood of an edge. Contrary to the common view, we argue
  that such {\em graph embeddings} do \emph{not} capture salient
  properties of complex networks. The two properties we focus on are
  low degree and large clustering coefficients, which have been widely
  established to be empirically true for real-world networks. We
  mathematically prove that any embedding (that uses dot products to
  measure similarity) that can successfully create these two
  properties must have rank nearly linear in the number of
  vertices. Among other implications, this establishes that popular
  embedding techniques such as Singular Value Decomposition and
  node2vec fail to capture significant structural aspects of
  real-world complex networks. Furthermore, we empirically study a
  number of different embedding techniques based on dot product, and
  show that they all fail to capture the triangle structure.
\end{abstract}



\section{Introduction} \label{sec:intro}

Complex networks (or graphs) are a fundamental object of study
in modern science, across domains as diverse as the social sciences,
biology, physics, computer science, and engineering~\cite{WaFa94,Ne-survey,EaKl-book}. 
Designing good models for these networks is a crucial area of research,
and also affects society at large, given the role of online social networks
in modern human interaction~\cite{BaAl99,WaSt98,ChFa06}. Complex networks are massive, high-dimensional,
discrete objects, and are challenging to work with in a modeling
context. A common method of dealing with this challenge
is to construct a low-dimensional Euclidean embedding that tries
to capture the structure of the network (see~\cite{HaYiLe17} for a recent survey). Formally,
we think of the $n$ vertices as vectors $\vec{v}_1, \vec{v}_2, \ldots, \vec{v}_n \in \RR^d$,
where $d$ is typically constant (or very slowly growing in $n$).
The likelihood of an edge $(i,j)$ is proportional to (usually a non-negative
monotone function in) $\vec{v}_i\cdot \vec{v}_j$~\cite{AhSh+13,CaLuXu16}. This gives a graph
distribution that the observed network is assumed to be generated from.

The most important method to get such embeddings is the Singular Value Decomposition (SVD)
or other matrix factorizations of the adjacency
matrix~\cite{AhSh+13}. Recently, there has also been an explosion of
interest in using methods from deep neural networks to learn such graph embeddings~\cite{PeAlSk14,TaQu+15,CaLuXu16,grover2016node2vec}
(refer to~\cite{HaYiLe17} for more references). Regardless of the
specific method, a key goal in building an embedding is to 
keep the dimension $d$ small --- while trying to preserve the network
structure --- as the embeddings are used in a variety of downstream
modeling tasks such as graph
clustering, nearest neighbor search, and link
prediction~\cite{Twitter-embeddings}.  Yet a fundamental question remains unanswered: 
to what extent do such low dimensional embeddings actually capture
the structure of a complex network?

These models are often justified by treating the (few) dimensions as ``interests" of individuals,
and using similarity of interests (dot product) to form edges.
Contrary to the dominant view, we argue that low-dimensional embeddings are
\emph{not} good representations of complex networks. We demonstrate
mathematically and empirically that they lose local structure,
one of the hallmarks of complex networks. This runs counter
to the ubiquitous use of SVD in data analysis. The weaknesses
of SVD have been empirically observed in recommendation tasks~\cite{BaChGo10,GuGo+13,KlUgKl17},
and our result provides a mathematical validation of these findings.

Let us define the setting formally. Consider a set of vectors $\vec{v}_1, \vec{v}_2, \ldots, \vec{v}_n \in \mathbb{R}^d$
(denoted by the $d \times n$ matrix $V$) used to represent the $n$ vertices in a network. 
Let $\cG_V$ denote the following distribution of graphs over the vertex set $[n]$.
For each index pair $i,j$, independently insert (undirected) edge $(i,j)$ with probability
$\max(0,\min(\vec{v}_i \cdot \vec{v}_j, 1))$. (If $\vec{v}_i \cdot \vec{v}_j$ is negative,
$(i,j)$ is never inserted. If $\vec{v}_i \cdot \vec{v}_j \geq 1$, $(i,j)$ is always inserted.)
We will refer to this model as the ``embedding" of a graph $G$, and focus
on this formulation in our theoretical results. 
{This is a standard model in the literature, 
and subsumes the classic Stochastic Block Model~\cite{HoLa83}
and Random Dot Product Model~\cite{YoSc07,AtFi+18}. There are 
alternate models that use different functions of the dot product
for the edge probability, which are discussed in Section~\ref{sec:variants}.}
Matrix factorization is a popular method to obtain such a vector representation: the original adjacency matrix $A$
is ``factorized" as $V^TV$, where the columns of $V$ are $\vec{v}_1, \vec{v}_2, \ldots, \vec{v}_n$.

\begin{figure}
        \includegraphics[scale=0.5]{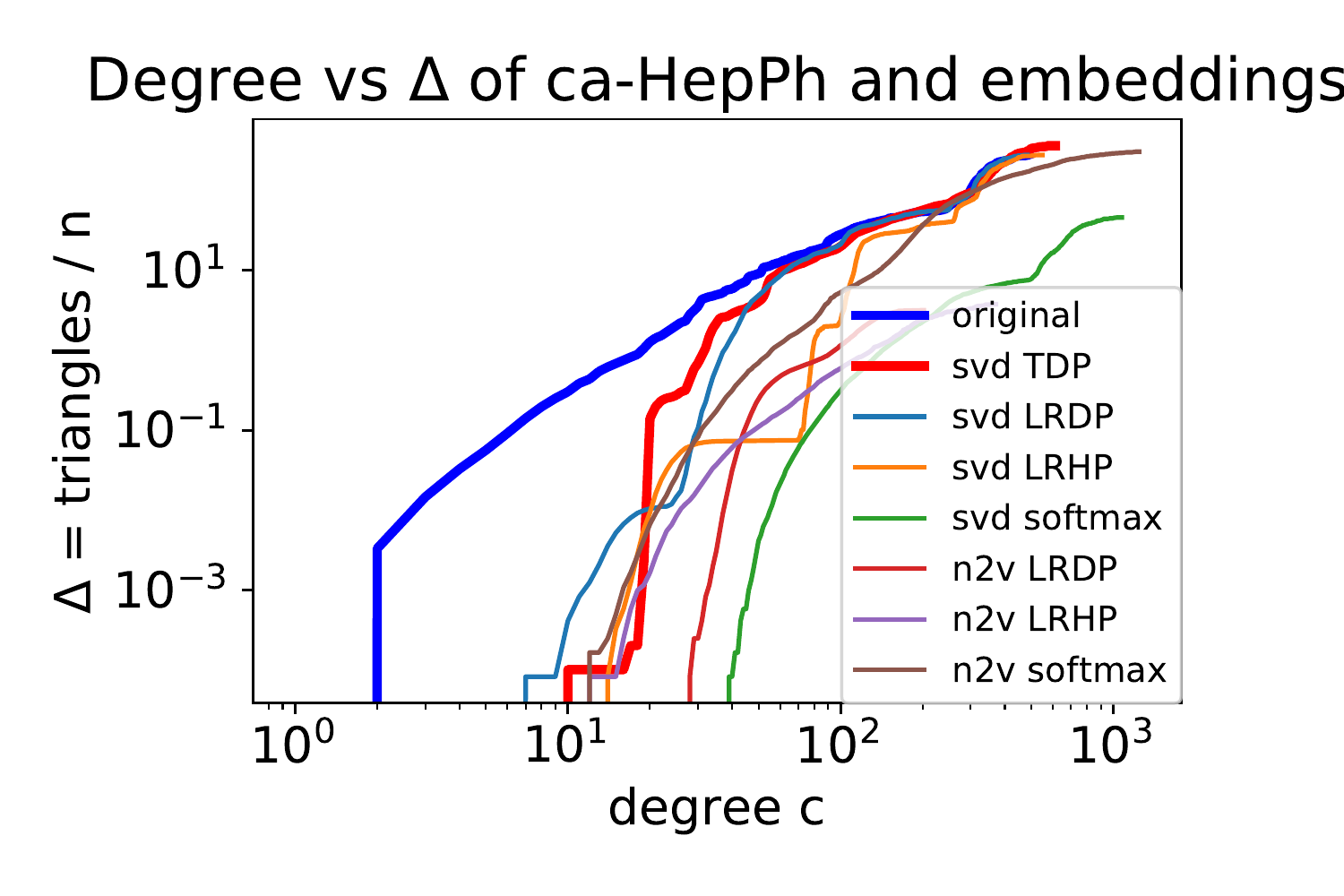}
				\captionof{figure}{\small Plots of degree $c$ vs $\Delta$: For a High Energy Physics
        coauthorship network,
        we plot $c$ versus the total number of triangles only involving vertices of
        degree at most $c$. We divide the latter by the total number of vertices $n$,
        so it corresponds to $\Delta$, as in Def.~\ref{def:foundation}. We plot these
        both for the original graph (in thick blue), and for a variety of embeddings (explained in Section~\ref{sec:variants}).
        For each embedding, we plot the maximum $\Delta$ in 
        a set of 100 samples from a 100-dimensional embedding. The embedding analyzed by
        our main theorem (TDP) is given in thick red. Observe how the embeddings generate
        graphs with very few triangles among low degree vertices. The gap in $\Delta$ for low degree is
        2-3 orders of magnitude. The other lines correspond to alternate embeddings,
        using the {\sc node2vec} vectors and/or different functions of the dot product.} \label{fig:intro-triangle-distro}
\end{figure}

Two hallmarks of real-world graphs are: (i) Sparsity: The average degree is typically constant with respect
to $n$, and (ii) Triangle density: there are many triangles incident to low degree vertices~\cite{WaSt98,SaCaWiZa10,SeKoPi11,DPKS12}.
The large number of triangles is considered a local manifestation of community structure. Triangle counts have a rich history
in the analysis and algorithmics of complex networks. Concretely, we measure these properties
simultaneously as follows. 

\begin{definition} \label{def:foundation} For parameters $c > 1$ and $\Delta > 0$, a graph $G$ with $n$ vertices
has a \emph{$(c,\Delta)$-triangle foundation} if there are at least $\Delta n$ triangles contained among vertices of degree at most $c$.
Formally, let $S_c$ be the set of vertices of degree at most $c$. Then, the number of triangles
in the graph induced by $S_c$ is at least $\Delta n$.
\end{definition}

Typically, we think of both $c$ and $\Delta$ as constants. 
{We emphasize that $n$ is the total number of vertices in $G$, not the number
of vertices in $S$ (as defined above).}
Refer to real-world graphs
in Table~\ref{tab:datasets}. In Figure~\ref{fig:intro-triangle-distro}, we plot the value
of $c$ vs $\Delta$. (Specifically, the $y$ axis is the number of triangles divided by $n$.)
This is obtained by simply counting the number of triangles contained in the set of vertices of degree at most $c$. 
Observe that for all graphs,
for $c \in [10,50]$, we get a value of $\Delta > 1$ (in many cases $\Delta > 10$). 

Our main result is that \emph{any} embedding of graphs that generates graphs with $(c,\Delta)$-triangle foundations,
with constant $c,\Delta$, must have near linear rank. This
contradicts the belief that low-dimensional embeddings capture the structure of real-world complex networks.

\begin{theorem} \label{thm:main} Fix $c > 4, \Delta > 0$. Suppose the expected number of triangles 
in $G \sim \cG_V$ that only involve vertices of expected degree $c$ is at least $\Delta n$.
Then, the rank of $V$ is at least $\min(1,\poly(\Delta/c)) n/\lg^2n$.
\end{theorem}

Equivalently, graphs generated from low-dimensional embeddings cannot contain many
triangles only on low-degree vertices. {We point out an important
implication of this theorem for Stochastic Block Models. In this model, each vertex is modeled
as a vector in $[0,1]^d$, where the $i$th entry indicates the likelihood
of being in the $i$th community. The probability of an edge is exactly
the dot product. In community detection applications, $d$ is thought of as a constant,
or at least much smaller than $n$. On the contrary, \Thm{main} implies
that $d$ must be $\Omega(n/\lg^2n)$ to accurately model the low-degree triangle behavior.}

\subsection{Empirical validation} \label{sec:emp}

We empirically validate the theory on a collection of complex
networks detailed in Table~\ref{tab:datasets}. For each real-world graph, we compute a 100-dimensional embedding through SVD
(basically, the top 100 singular vectors of the adjacency matrix).
We generate $100$ samples of graphs from these embeddings, and compute
their $c$ vs $\Delta$ plot. This is plotted with the true $c$ vs $\Delta$ plot.
(To account for statistical variation, we plot
the \emph{maximum} value of $\Delta$ observed in the samples, over all graphs. The variation
observed was negligible.)
Figure~\ref{fig:intro-triangle-distro} shows such a plot for a physics coauthorship network.
More results are given in Section~\ref{sec:emp-details}.
 
Note that this plot
is significantly off the mark at low degrees for the embedding. Around the lowest degree,
the value of $\Delta$ (for the graphs generated by the embedding) 
is 2-3 order of magnitude smaller than the original value. This demonstrates
that the local triangle structure is destroyed around low degree vertices.
Interestingly, the total number of triangles is preserved well, as shown
towards the right side of each plot. Thus, a nuanced view of the triangle
distribution, as given in \Def{foundation}, is required to see the shortcomings
of low dimensional embeddings.

\subsection{Alternate models} \label{sec:variants}
{We note that several other functions of dot product have been
  proposed in the literature, such as the softmax
  function~\cite{PeAlSk14,grover2016node2vec} and linear 
models of the dot product~\cite{HaYiLe17}. Theorem~\ref{thm:main} does not have
direct implications for such models, but our empirical validation holds for 
them as well. The embedding in Theorem~\ref{thm:main} uses the \emph{truncated
dot product} (TDP) function $max(0,\min(\vec{v}_i \cdot \vec{v}_j, 1))$ to model
edge probabilities. We construct other embeddings that compute edge probabilities
using machine learning models with the dot product and Hadamard product
as features. This subsumes linear models as given in~\cite{HaYiLe17}. Indeed, the TDP
can be smoothly approximated as a logistic function.
We also consider (scaled) softmax functions, as in~\cite{PeAlSk14}, and standard
machine learning models (LRDP, LRHP). 
(Details about these models are given in Section~\ref{sec:emp-details}\ref{sec:alt}.)

For each of these models (softmax, LRDP, LRHP), we perform the same experiment described above.
Figure~\ref{fig:intro-triangle-distro} also shows the plots for these other models.
Observe that \emph{none} of them capture the low-degree triangle structure,
and their $\Delta$ values are all 2-3 orders of magnitude lower than the original.  

In addition (to the extent possible), we compute vector embeddings
from a recent deep learning based method (node2vec~\cite{grover2016node2vec}).
We again use all the edge probability models discussed above, and perform
an identical experiment (in Figure~\ref{fig:intro-triangle-distro}, these
are denoted by ``n2v"). Again, we observe that the low-degree triangle behavior
is not captured by these deep learned embeddings.
}

\subsection{Broader context} \label{sec:context}
{The use of geometric embeddings for graph analysis has a rich history, arguably
going back to spectral clustering~\cite{Fi73}. In recent years, the Stochastic Block Model
has become quite popular in the statistics and algorithms community~\cite{HoLa83}.
and the Random Dot Product Graph model is a generalization of this notion
(refer to recent surveys~\cite{Ab18,AtFi+18}). As mentioned earlier, Theorem~\ref{thm:main}
brings into question the standard uses of these methods to model social networks.
The use of vectors to represent vertices
is sometimes referred to as \emph{latent space models}, where geometric proximity models
the likelihood of an edge. Although dot products are widely used, we
note that some classic latent space approaches use
Euclidean distance (as opposed to dot product) to model edge probabilities~\cite{HoRa+02},
and this may avoid the lower bound of Theorem~\ref{thm:main}.
Beyond graph analysis, the method of Latent Semantic Indexing (LSI) also falls
in the setting of Theorem~\ref{thm:main}, wherein we have a low-dimensional embedding
of ``objects" (like documents) and similarity is measured by dot product~\cite{lsi-wiki}. 
}

\section{High-level description of the proof} \label{sec:highlevel} 

In this section, we sketch the proof of Theorem~\ref{thm:main}.
{The sketch provides
sufficient detail for a reader who wants to understand the reasoning
behind our result, but is not concerned with technical details. 
We will make the simplifying assumption that all $v_i$s have the same
length $L$.  We note that this setting is interesting in
its own right, since it is often the case in practice that all vectors
are non-negative and normalized. In this case, we get a stronger rank
lower bound that is linear in $n$. Section~\ref{sec:vary} 
provides intuition on how we can remove this assumption. 
The full details of the proof are given in Section~\ref{sec:full-proof}.
}

First, we lower bound $L$. By Cauchy-Schwartz, $\vec{v}_i \cdot \vec{v}_j \leq L^2$. Let $X_{i,j}$
be the indicator random variable for the edge $(i,j)$ being present. Observe that all $X_{i,j}$s
are independent and $\EX[X_{i,j}] = \min(\vec{v}_i \cdot \vec{v}_j, 1) \leq L^2$.

The expected number of triangles in $G \sim \cG_V$ is:
\begin{eqnarray} 
& & \EX[\sum_{i\neq j \neq k} X_{i,j} X_{j,k} X_{i,k}] \\
& \leq & \sum_i \sum_{j,k} \EX[X_{j,k}] \EX[X_{i,j}] \EX[X_{i,k}] \\
& \leq & L^2 \sum_i \sum_{j,k} \EX[X_{i,j}] \EX[X_{i,k}] = L^2 \sum_i (\sum_j \EX[X_{i,j}])^2
\label{eq:tri1}
\end{eqnarray}
Note that $\sum_j \EX[X_{i,j}] = \EX[\sum_j X_{i,j}]$ is at most the degree of $i$,
{which is at most $c$. (Technically, the $X_{i,i}$ term creates a self-loop, so the correct
upper bound is $c+1$. For the sake of cleaner expressions, we omit the additive $+1$
in this sketch.)}
 
 The expected number of triangles is at least $\Delta n$. Plugging these bounds
in:
\begin{equation} 
\Delta n \leq L^2 c^2 n \Longrightarrow L \geq \sqrt{\Delta}/c \label{eq:tri2}
\end{equation}
Thus, the vectors have length at least $\sqrt{\Delta}/c$. Now, we lower bound the rank of $V$.
It will be convenient to deal with the Gram matrix $M = V^TV$, which has the same rank as $V$.
Observe that $M_{i,j} = \vec{v}_i \cdot \vec{v}_j \leq L^2$.
We will use the following lemma stated first by Swanapoel, but has appeared
in numerous forms previously~\cite{Sw14}. 

\begin{lemma} [Rank lemma] Consider any square matrix $M \in \RR^{n \times n}$. Then
$$ \rank(M) \ge \frac{| \sum_i M_{i,i}|^2}{\left( \sum_i \sum_j |M_{i,j}|^2 \right)} $$
\end{lemma}

Note that $M_{i,i} = \vec{v}_i \cdot \vec{v}_i = L^2$, so the numerator
$|\sum_i M_{i,i}|^2 = n^2 L^4$. The denominator requires more work. We split 
it into two terms.
\begin{equation} \label{eq:small-denom}
\sum_{\substack{i,j\\\vec{v}_i \cdot \vec{v}_j \leq 1}} (\vec{v}_i \cdot \vec{v}_j)^2
\leq \sum_{\substack{i,j\\\vec{v}_i \cdot \vec{v}_j \leq 1}} \vec{v}_i \cdot \vec{v}_j \leq cn
\end{equation}
If for $i\neq j$, $\vec{v}_i \cdot \vec{v}_j > 1$, then $(i,j)$ is an edge with probability $1$.
Thus, there can be at most $(c-1)n$ such pairs. Overall, there are at most $cn$ pairs 
such that $\vec{v}_i \cdot \vec{v}_j > 1$. So, $\sum_{\substack{i,j\\ \vec{v}_i \cdot \vec{v}_j > 1}} (\vec{v}_i \cdot \vec{v}_j) 
\leq cnL^4$. Overall, we lower bound the denominator in the rank lemma by $cn(L^4+1)$.

We plug these bounds into the rank lemma. We use the fact that $f(x) = x/(1+x)$ is decreasing for positive $x$,
and that $L \geq \sqrt{\Delta}/c$.
\begin{eqnarray*}
\rank(M) \geq \frac{n^2 L^4}{cn(L^4 + 1)} \geq \frac{n}{c} \cdot \frac{\Delta^2/c^4}{\Delta^2/c^4 + 1}
= \frac{\Delta^2}{c(\Delta^2 + c^4)} \cdot n
\end{eqnarray*}

\subsection{Dealing with varying lengths} \label{sec:vary}

{The math behind \Eqn{tri2} still holds with
the right approximations. Intuitively, the existence of at least $\Delta n$ triangles 
implies that a sufficiently large number of vectors have length at least $\sqrt{\Delta}/c$.
On the other hand, these long vectors need to be ``sufficiently far
away" to ensure that the vertex degrees remain low. There
are many such long vectors, and they can can only be far away 
when their dimension/rank is sufficiently high.

The rank lemma is the main technical tool that formalizes this
intuition. When vectors are of varying length,
the primary obstacle is the presence of extremely long vectors that create triangles.
The numerator in the rank lemma sums $M_{i,i}$, which is the length of the vectors.
A small set of extremely long vectors could dominate the sum, increasing the numerator.
In that case, we do not get a meaningful rank bound.

But, because the vectors inhabit low-dimensional space, the long vectors from \emph{different}
clusters interact with each other. We prove a ``packing" lemma (\Lem{dot}) showing that 
there must be many large positive dot products among a set of extremely long vectors.
Thus, many of the corresponding vertices have large degree, and triangles incident
to these vertices do not contribute to low degree triangles.
Operationally, the main proof uses the packing lemma to show that there are few long vectors.
These can be removed without affecting the low degree structure. One can then perform
a binning (or ``rounding") of the lengths of the remaining vectors, to implement
the proof described in the above section.
}

\section{Proof of Theorem~\ref{thm:main}} \label{sec:full-proof}

For convenience, we restate the setting.
Consider a set of vectors $\vec{v}_1, \vec{v}_2, \ldots, \vec{v}_n$ $\in \RR^d$, that represent
the vertices of a social network. 
We will also use the matrix $V \in \RR^{d \times n}$ for these vectors,
where each column is one of the $\vec{v}_i$s.
Abusing notation, we will use $V$ to represent
both the set of vectors as well as the matrix. We will refer to the vertices by the index in $[n]$.

Let $\cG_V$ denote the following distribution of graphs over the vertex set $[n]$.
For each index pair $i,j$, independently insert (undirected) edge $(i,j)$ with probability
$\max(0,\min(\vec{v}_i \cdot \vec{v}_j, 1))$.

\subsection{The basic tools} \label{sec:tools}

We now state some results that will be used in the final proof.

\begin{lemma} \label{lem:rank} [Rank lemma~\cite{Sw14}] Consider any square matrix $A \in \RR^{n \times n}$. Then
$$ | \sum_i A_{i,i}|^2 \leq \rank(A) \left( \sum_i \sum_j |A_{i,j}|^2 \right) $$
\end{lemma}

\begin{lemma} \label{lem:neg} Consider a set of $s$ vectors $\vec{w}_1, \vec{w}_2, \ldots, \vec{w}_s$ in $\RR^d$.
$$ \sum_{\substack{(i,j) \in [s] \times [s] \\ \vec{w}_i \cdot \vec{w}_j < 0}} |\vec{w}_i \cdot \vec{w}_j|
\leq \sum_{\substack{(i,j) \in [s] \times [s] \\ \vec{w}_i \cdot \vec{w}_j > 0}} |\vec{w}_i \cdot \vec{w}_j| $$
\end{lemma}

\begin{proof}  Note that $(\sum_{i \leq s} \vec{w}_i)\cdot (\sum_{i \leq s} \vec{w}_i) \geq 0$.
Expand and rearrange to complete the proof.
\end{proof}

Recall that an independent set is a collection of vertices that induce no edge.

\begin{lemma} \label{lem:ind} Any graph with $h$ vertices and maximum degree $b$
has an independent set of at least $h/(b+1)$.
\end{lemma}

\begin{proof}  Intuitively, one can incrementally build an independent set,
by adding one vertex to the set, and removing at most $b+1$ vertices from the graph.
This process can be done at least $h/(b+1)$ times.

Formally, we prove by induction on $h$. First we show
the base case. If $h \leq b+1$, then the statement is trivially true.
(There is always an independent set of size $1$.) For the induction step, let us construct
an independent set of the desired size. Pick an arbitrary vertex $x$ and add it to the independent set.
Remove $x$ and all of its neighbors. By the induction hypothesis, the remaining graph
has an independent set of size at least $(h - b - 1)/(b+1) = h/(b+1) - 1$.
\end{proof}

\begin{claim} \label{clm:var} Consider the distribution $\cG_{V}$.
Let $\degr_i$ denote the degree of vertex $i \in [n]$.
$\EX[\degr^2_i] \leq \EX[\degr_i] + \EX[\degr_i]^2$.
\end{claim}

\begin{proof} (of Claim~\ref{clm:var}) Fix any vertex $i \in [n]$. Observe that $\degr_i = \sum_{j \neq i} X_j$, where $X_j$ is the indicator random variable
for edge $(i,j)$ being present. Furthermore, all the $X_j$s are independent.
\begin{eqnarray*}
    \EX[D^2_i] & = & \EX[(\sum_{j \neq i} X_j)^2] = \EX[\sum_{j \neq i} X^2_j + 2\sum_{j\neq j'} X_j X_{j'}] \\
    & = & \EX[\sum_{j \neq i} X_j] + 2\sum_{j \neq j'} \EX[X_j] \EX[X_{j'}] \\
    & \leq & \EX[\degr_i] + (\sum_{j \neq i} \EX[X_j])^2 = \EX[\degr_i] + \EX[\degr_i]^2
\end{eqnarray*}
\vspace{-5pt}
\end{proof}

A key component of dealing with arbitrary length vectors is the following dot product lemma.
This is inspired by results of Alon~\cite{Al03} and Tao~\cite{Ta-post},
who get a stronger lower bound of $1/\sqrt{d}$ for \emph{absolute values} of the dot products.

\begin{lemma}  \label{lem:dot} Consider any set of $4d$ unit vectors $\vec{u}_1, \vec{u}_2, \ldots, \vec{u}_{4d}$
in $\RR^d$. There exists some $i \neq j$ such that $\vec{u}_i \cdot \vec{u}_j \geq 1/4d$.
\end{lemma}

\begin{proof} (of \Lem{dot})  We prove by contradiction, so assume $\forall i \neq j, \vec{u}_i \cdot \vec{u}_j < 1/4d$.
We partition the set $[4d] \times [4d]$ into 
$\cN = \{(i,j) | \vec{u}_i \cdot \vec{u}_j < 0\}$ and $\cP = \{(i,j) | \vec{u}_i \cdot \vec{u}_j \geq 0\}$.
The proof goes by providing (inconsistent) upper and lower bounds for $\sum_{(i,j) \in \cN} |\vec{u}_i \cdot \vec{u}_j|^2$.
First, we upper bound $\sum_{(i,j) \in \cN} |\vec{u}_i \cdot \vec{u}_j|^2$ by:
\begin{eqnarray}
    & \leq &  \sum_{(i,j) \in \cN} |\vec{u}_i \cdot \vec{u}_j| \ \ \ \ \ \ \ \textrm{($\vec{u}_i$s are unit vectors)} \nonumber \\
    & \leq & \sum_{i \leq 4d} \|\vec{u}_i\|^2_2 + \sum_{\substack{1 \leq i \neq j \leq 4d\\ (i,j) \in \cP}} |\vec{u}_i \cdot \vec{u}_j| \ \ \ \ \ \ \ \textrm{(\Lem{neg})} \nonumber \\
    & < & 4d + 16d^2/4d = 8d \ \ \ \ \textrm{(since $\vec{u}_i \cdot \vec{u}_j < 1/4d$)} \label{eq:dot-upper}
\end{eqnarray}
For the lower bound, we invoke the rank bound of \Lem{rank} on the $4d \times 4d$ Gram matrix $M$ of $\vec{u}_1, \ldots, \vec{u}_{4d}$.
Note that $\rank(M) \leq d$, $M_{i,i} = 1$, and $M_{i,j} = \vec{u}_i \cdot \vec{u}_j$.
By \Lem{rank}, $\sum_{(i,j) \in [4d] \times [4d]} |\vec{u}_i \cdot \vec{u}_j|^2 \geq (4d)^2/d = 16 d$. 
We bound
\begin{eqnarray}
\sum_{(i,j) \in \cP} |\vec{u}_i \cdot \vec{u}_j|^2 & = & \sum_{i \leq 4d} \|\vec{u}_i\|^2_2 + \sum_{(i,j) \in \cP, i \neq j} |\vec{u}_i \cdot \vec{u}_j|^2 \nonumber \\
& \leq & 4d + (4d)^2/(4d)^2 \leq 5d
\end{eqnarray}
Thus, $\sum_{(i,j) \in \cN} |\vec{u}_i \cdot \vec{u}_j|^2 \geq 16d - 5d = 11d$. This contradicts the bound of \Eqn{dot-upper}.

\end{proof} 

\subsection{The main argument} \label{sec:proof}

We prove by contradiction. We assume that the expected number
of triangles contained in the set of vertices of expected degree at most $c$,
is at least $\Delta n$. {We remind the reader that $n$
is the total number of vertices.} For convenience, we simply remove the vectors
corresponding to vertices with expected degree at least $c$.
Let $\hat{V}$ be the matrix of the remaining vectors, and we focus
on $\cG_{\hat{V}}$. The expected number of triangles in $G \sim \cG_{\hat{V}}$
is at least $\Delta n$.

The overall proof can be thought of in three parts.

{\em Part 1, remove extremely long vectors:} Our final aim is to use the rank lemma (Lemma~\ref{lem:rank})
to lower bound the rank of $V$. The first problem we encounter is that
extremely long vectors can dominate the expressions in the rank lemma,
and we do not get useful bounds. We show that the number of such long vectors
is extremely small, and they can removed without affecting too many triangles.
In addition, we can also remove extremely small vectors, since they cannot participate
in many triangles.

{\em Part 2, find a ``core" of sufficiently long vectors that contains enough triangles:} The previous step
gets a ``cleaned" set of vectors. Now, we bucket these vectors by length. We show
that there is a large bucket, with vectors that are sufficiently long, such that
there are enough triangle contained in this bucket.

{\em Part 3, apply the rank lemma to the ``core":} We now focus on this core of vectors,
    where the rank lemma can be applied. 
    \main{At this stage, the mathematics shown
    in Section~\ref{sec:highlevel} can be carried out almost directly.}

Now for the formal proof.
For the sake of contradiction, we assume that $d = \rank(\hat{V}) < \alpha (\Delta^4/c^9) \cdot n/\lg^2n$
(for some sufficiently small constant $\alpha > 0$).

{\bf Part 1: Removing extremely long (and extremely short) vectors}

We begin by showing that there cannot be many long vectors in $\hat{V}$.

\begin{lemma} \label{lem:long} There are at most $5cd$ vectors of length at least $2\sqrt{n}$.
\end{lemma}

\begin{proof} Let $\cL$ be the set of ``long" vectors, those with length at least $2\sqrt{n}$.
Let us prove by contradiction, so assume there are more than $5cd$ long vectors.
Consider a graph $H = (\cL,E)$, where vectors $\vec{v_i},\vec{v_j} \in \cL$ ($i \neq j$) are connected
by an edge if $\frac{\vec{v_i}}{\|\vec{v}_i\|_2} \cdot \frac{\vec{v_j}}{\|\vec{v}_j\|_2} \geq 1/4n$.
We choose the $1/4n$ bound to ensure that all edges in $H$ are edges in $G$.

Formally, for any edge $(i,j)$ in $H$, $\vec{v_i}\cdot\vec{v_j} \geq \|\vec{v}_i\|_2 \|\vec{v}_j\|_2/4n 
\geq (2\sqrt{n})^2/4n = 1$. So $(i,j)$ is an edge with probability $1$
in $G \sim \cG_V$. The degree of any vertex
in $H$ is at most $c$. By \Lem{ind}, $H$ contains an independent set $I$ of size at least $5cd/(c+1) \geq 4d$.
Consider an arbitrary sequence of $4d$ (normalized) vectors in $I$  $\vec{u}_1, \ldots, \vec{u}_{4d}$.
Applying \Lem{dot} to this sequence, we deduce the existence of $(i,j)$ in $I$ ($i \neq j$) such that 
$\frac{\vec{v_i}}{\|\vec{v}_i\|_2} \cdot \frac{\vec{v_j}}{\|\vec{v}_j\|_2} \geq 1/4d \geq 1/4n$.
Then, the edge $(i,j)$ should be present in $H$, contradicting the fact that $I$ is an independent set.
\end{proof}

Denote by $V'$ the set of all vectors in $\hat{V}$ with length in the range $[n^{-2}, 2\sqrt{n}]$. 

\begin{claim} \label{clm:prune} The expected degree of every vertex in $G \sim \cG_{V'}$ is 
at most $c$, and the expected number of triangles in $G$ is at least $\Delta n/2$.
\end{claim}

\begin{proof} Since removal of vectors can only decrease the 
degree, the expected degree of every vertex in $\cG_{V'}$ is naturally at most $c$.
It remains to bound the expected number of triangles in $G \sim \cG_{V'}$.
By removing vectors in $V\setminus V'$, we potentially lose some triangles. Let us
categorize them into those that involve at least one ``long" vector (length $\geq 2\sqrt{n}$)
and those that involve at least one ``short" vector (length $\leq n^{-2}$) but no long vector.

We start with the first type.
By \Lem{long}, there are at most $5cd$ long vectors. For any vertex, the expected number
of triangles incident to that vertex is at most the expected square of the degree.
By \Clm{var}, the expected degree squares is at most $c + c^2 \leq 2c^2$.  Thus,
the expected total number of triangles of the first type is at most $5cd\times 2c^2 \leq \Delta n/\lg^2n$.

Consider any triple of vectors $(\vec{u},\vec{v},\vec{w})$ where $\vec{u}$ is short
and neither of the others are long. The probability that this triple forms
a triangle is at most 
\begin{eqnarray*}
& & \min(\vec{u}\cdot\vec{v},1) \cdot \min(\vec{u}\cdot\vec{w},1) \\
& \leq & \min(\|\vec{u}\|_2\|\vec{v}\|_2,1) \cdot \min(\|\vec{u}\|_2 \|\vec{w}\|_2,1) \\
& \leq & (n^{-2} \cdot 2\sqrt{n})^2 \leq 4n^{-3} 
\end{eqnarray*}
Summing over all such triples, the expected number of such triangles is at most $4$.

Thus, the expected number of triangles in $G \sim \cG_{V'}$ is at least $\Delta n - \Delta n/\lg^2n - 4 \geq \Delta n/2$.
\end{proof}

{\bf Part 2: Finding core of sufficiently long vectors with enough triangles}

For any integer $r$, let $V_r$ be the set of vectors $\{\vec{v} \in V' | \|\vec{v}\|_2 \in [2^r, 2^{r+1})\}$.
Observe that the $V_r$s form a partition of $V'$. Since all lengths in $V'$
are in the range $[n^{-2},2\sqrt{n}]$, there are at most $3\lg n$ non-empty $V_r$s.
Let $R$ be the set of indices $r$ such that $|V_r| \geq (\Delta/60c^2)(n/\lg n)$.
Furthermore, let $V''$ be $\bigcup_{r\in R} V_r$.

\begin{claim} \label{clm:prune2} The expected number of triangles in $G \sim \cG_{V''}$ is at least $\Delta n/8$.
\end{claim}

\begin{proof} The total number of vectors in $\bigcup_{r \notin R} V_r$ is at most 
$3\lg n \times (\Delta/60c^2) (n/\lg n)$  $\leq (\Delta/20c^2) n$. By \Clm{var} and linearity of expectation,
the expected sum of squares of degrees of all vectors in $\bigcup_{r \notin R} V_r$
is at most $(d+c^2) \times (\Delta/20c^2) n$ $\leq \Delta n/10$. Since the expected
number of triangles in $G \sim \cG_{V'}$ is at least $\Delta n/2$
(\Clm{prune}) and the expected number of triangles
incident to vectors in $V' \setminus V''$ is at most $\Delta n/10$,
the expected number of triangles in $G \sim \cG_{V''}$ is at least $\Delta n/2 - \Delta n/10 \geq \Delta n/8$.
\end{proof}

We now come to an important claim. Because the expected number of triangles in $G \sim \cG_{V''}$ is large,
we can prove that $V''$ must contain vectors of at least constant length.

\begin{claim} \label{clm:len} $\max_{r \in R} 2^r \geq \sqrt{\Delta}/4c$.
\end{claim}

\begin{proof} Suppose not. Then every vector in $V''$ has length at most $\sqrt{\Delta}/4c$.
By Cauchy-Schwartz, for every pair $\vec{u}, \vec{v} \in V''$, $\vec{u}\cdot\vec{v} \leq \Delta/16c^2$.
Let $I$ denote the set of vector indices in $V''$ (this corresponds to the vertices in $G \sim \cG_{V''}$).
For any two vertices $i \neq j \in I$, let $X_{i,j}$ be the indicator random variable for 
edge $(i,j)$ being present.
The expected number of triangles incident to vertex $i$ in $G \sim \cG_{V''}$
is 
$$ \EX[\sum_{j\neq k \in I} X_{i,j}X_{i,k}X_{j,k}]  =  \sum_{j\neq k \in I} \EX[X_{i,j} X_{i,k}] \EX[X_{j,k}] $$
Observe that $\EX[X_{j,k}]$ is at most $\vec{v_j} \cdot \vec{v_k} \leq \Delta/16c^2$.
Furthermore, $\sum_{j \neq k \in I} \EX[X_{i,j} X_{i,k}] = \EX[D^2_i]$
{(recall that $D_i$ is the degree of vertex $i$.)}
By \Clm{var}, this is at most $c+c^2 \leq 2c^2$. The expected number of triangles in $G \sim \cG_{V''}$
is at most $n \times 2c^2 \times \Delta/16c^2 = \Delta n/8$. This contradicts \Clm{prune2}.
\end{proof}

{\bf Part 3: Applying the rank lemma to the core}

We are ready to apply the rank bound of \Lem{rank} to prove the final result. The following
lemma contradicts our initial bound on the rank $d$, completing the proof. We will omit
some details in the following proof, and provide
a full proof in the SI.

\begin{lemma} \label{lem:main} $\rank(V'') \geq (\alpha \Delta^4/c^9) n/\lg^2 n$.
\end{lemma}

\begin{proof} It is convenient to denote the index set of $V''$ be $I$.
Let $M$ be the Gram Matrix $(V'')^T(V'')$, so for $i,j \in I$, 
$M_{i,j} = \vec{v}_i \cdot \vec{v}_j$
By \Lem{rank}, $\rank(V'') = \rank(M) \geq (\sum_{i \in I} M_{i,i})^2/\sum_{i,j \in I} |M_{i,j}|^2$.
Note that $M_{i,i}$ is $\|\vec{v_i}\|^2_2$, which is at least $2^{2r}$
for $\vec{v_i} \in V_r$. Let us denote $\max_{r \in R} 2^r$ by $L$,
so all vectors in $V''$ have length at most $2L$.
By Cauchy-Schwartz, all entries in $M$ are at most $4L^2$.

We lower bound the numerator. 
\begin{eqnarray*}
    & & \big(\sum_{i \in I} \|\vec{v_i}\|^2_2\big)^2 \geq \big(\sum_{r \in R} 2^{2r} |V_r|\big)^2 \\
    & \geq & \big(\max_{r \in R} 2^{2r} (\Delta/60 c^2)(n/\lg n)\big)^2 \\
    & = & L^4(\Delta^2/3600 c^4)(n^2/\lg^2 n)
\end{eqnarray*}

Now for the denominator. We split the sum into four parts and bound each separately.
\begin{eqnarray}
    \sum_{i,j \in I} |M_{i,j}|^2 & = & \sum_{i \in I} |M_{i,i}|^2 + \sum_{\substack{i,j \in I \\ i \neq j, M_{i,j} \in [0,1]}} |M_{i,j}|^2 \nonumber \\
    & + & \sum_{\substack{i,j \in I \\ i \neq j, M_{i,j} > 1}} |M_{i,j}|^2 + \sum_{\substack{i,j \in I \\ M_{i,j} < 0}} |M_{i,j}|^2 \nonumber :! 
\end{eqnarray}
Since $|M_{i,i}| \leq L^2$, the first term is at most $4nL^4$.
For $i \neq j$ and $M_{i,j} \in [0,1]$, the probability that edge $(i,j)$ is present is precisely $M_{i,j}$.
Thus, for the second term, 
\begin{equation}
\sum_{\substack{i,j \in I \\ i \neq j, M_{i,j} \in [0,1]}} |M_{i,j}|^2 
\leq \sum_{\substack{i,j \in I \\ i \neq j, M_{i,j} \in [0,1]}} M_{i,j}
\leq 2cn
\end{equation}
For the third term, we observe that when $M_{i,j} > 1$ (for $i \neq j$),
then $(i,j)$ is an edge with probability $1$. There can be at most $2cn$ pairs $(i,j)$, $i \neq j$,
such that $M_{i,j} > 1$. Thus, the third term is at most $2cn \cdot (4L^2)^2 = 32cnL^4$.

Now for the fourth term. Note that $M$ is a Gram matrix, so we can invoke \Lem{neg} on its entries.
\begin{eqnarray}
    \sum_{\substack{i,j \in I \\ M_{i,j} < 0}} |M_{i,j}|^2  
    & \leq & L^2\sum_{\substack{i,j \in I \\ M_{i,j} < 0}} |M_{i,j}| \nonumber \\
    & \leq & L^2(\sum_{i \in I} |M_{i,i}| + \sum_{\substack{i,j \in I \\ M_{i,j} > 0}} |M_{i,j}|) \nonumber \\
    & \leq & 4nL^4 + L^2 \sum_{\substack{i,j \in I \\ M_{i,j} \in [0,1]}} |M_{i,j}|  + 
4L^4 \sum_{\substack{i,j \in I \\ M_{i,j} > 1}} 1 \nonumber \\
    & \leq & 4nL^4 + 2cn L^2 + 8cnL^4
\end{eqnarray}
Putting all the bounds together,
we get that $\sum_{i,j \in I} |M_{i,j}|^2 \leq n(4L^4 + 2c + 32cL^4 + 4L^4 + 2cL^2 + 8cL^4)
\leq 32n(L^4 + c(1 + L^2 + L^4))$. If $L \leq 1$, we can upper bound
by $128cn$. If $L \geq 1$, we can upper bound by $128cnL^4$. In either case,
$128cn(1+L^4)$ is a valid upper bound.

%

Crucially, by \Clm{len}, $L \geq \sqrt{\Delta}/4c$. Thus, $4^4 c^4 L^4/\Delta^2 \geq 1$.
Combining all the bounds (and setting $\alpha < 1/(128\cdot 3600 \cdot 4^4)$),
\begin{eqnarray*}
    \rank(V'') & \geq & \frac{L^4(\Delta^2/3600 c^4)(n^2/\lg^2 n)}{128cn(1+16L^4)} \\
    & \geq &
\frac{L^4(\Delta^2/3600 c^4)(n/\lg^2 n)}{128cn(4^4 c^4 L^4/\Delta^2+16L^4)} \\
& \geq & (\alpha \Delta^4/c^9)(n/\lg^2 n) \nonumber
\end{eqnarray*}
\end{proof}

\section{Details of empirical results} \label{sec:emp-details}
	
 {\bf Data Availability:}   The datasets used are summarized in \Tab{datasets}. 
   We present here four publicly available datasets from different domains. The \texttt{ca-HepPh} is a co-authorship network
   \texttt{Facebook} is a social network, \texttt{cit-HepPh} is a citation network, all obtained from the SNAP graph database~\cite{Snap}.
   The \texttt{String\_hs} dataset is a protein-protein-interaction network obtained from~\cite{string}. (The citations
   provide the link to obtain the corresponding datasets.)

   {We first describe the primary experiment, used to validate Theorem~\ref{thm:main}
   on the SVD embedding.}
	We generated a $d$-dimensional embedding for various values of $d$
    using the SVD.
    Let $G$ be a graph with the $n \times n$ (symmetric) adjacency matrix $A$,
    with eigendecomposition $\Psi \Lambda \Psi^T$.  Let $\Lambda_d$ be the matrix with the $d \times d$ diagonal matrix 
    with the $d$ largest magnitude eigenvalues of $A$ along the diagonal. Let $\Psi_d$ be the $n \times d$ matrix 
    with the corresponding eigenvectors as columns. We compute the matrix $A_d = \Psi_d \Lambda_d \Psi_d^T$ and 
    refer to this as the $d$ spectral embedding of $G$. This is the standard PCA approach.

	From the spectral embeddings, we generate a graph from $A_d$ by considering every pair of vertices $(i, j)$ and generate a random value in $[0, 1]$. If the $(i, j)^{\textrm{th}}$ entry of $A_d$ is greater than the random value generated, the edge is added to the graph. Otherwise the edge is not present. This is the same as taking $A_d$ and setting all negative values to 0, and all values greater than 1 to 1 and performing Bernoulli trials for each edge with the resulting probabilities. 
    In all the figures, this is referred to as the ``SVD TDP" (truncated dot product) embedding.
		
\begin{figure*}
\begin{center}
\begin{tabular}{|c |c |c |c|} 
	\hline
		\textbf{Dataset name} & \textbf{Network type} & \textbf{Number of nodes} & \textbf{Number of edges} \\
		\hline
		Facebook \cite{Snap} & Social network & 4K & 88K \\
		\hline
		cit-HePh \cite{Arnet-data}& Citation & 34K & 421K \\
        \hline
		String\_hs \cite{string} & PPI & 19K & 5.6M \\
		\hline
		ca-HepPh \cite{Snap}& Co-authorship & 12K & 118M \\
		\hline
\end{tabular}
\captionof{table}{\small Table of datasets used}
\label{tab:datasets}
\end{center}
\end{figure*}

\begin{figure*}[h!]
        \includegraphics[scale=.37]{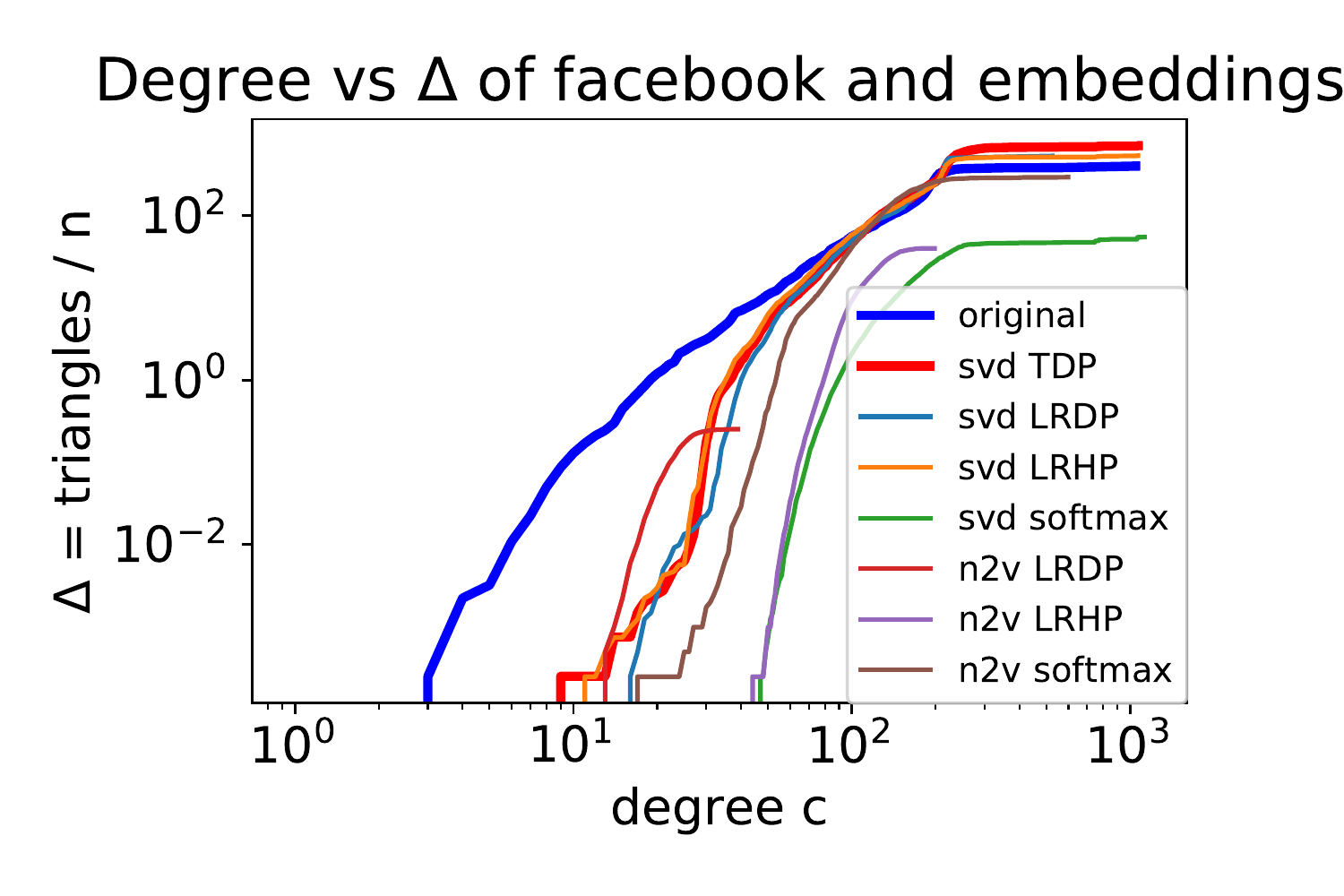}
        \includegraphics[scale=.37]{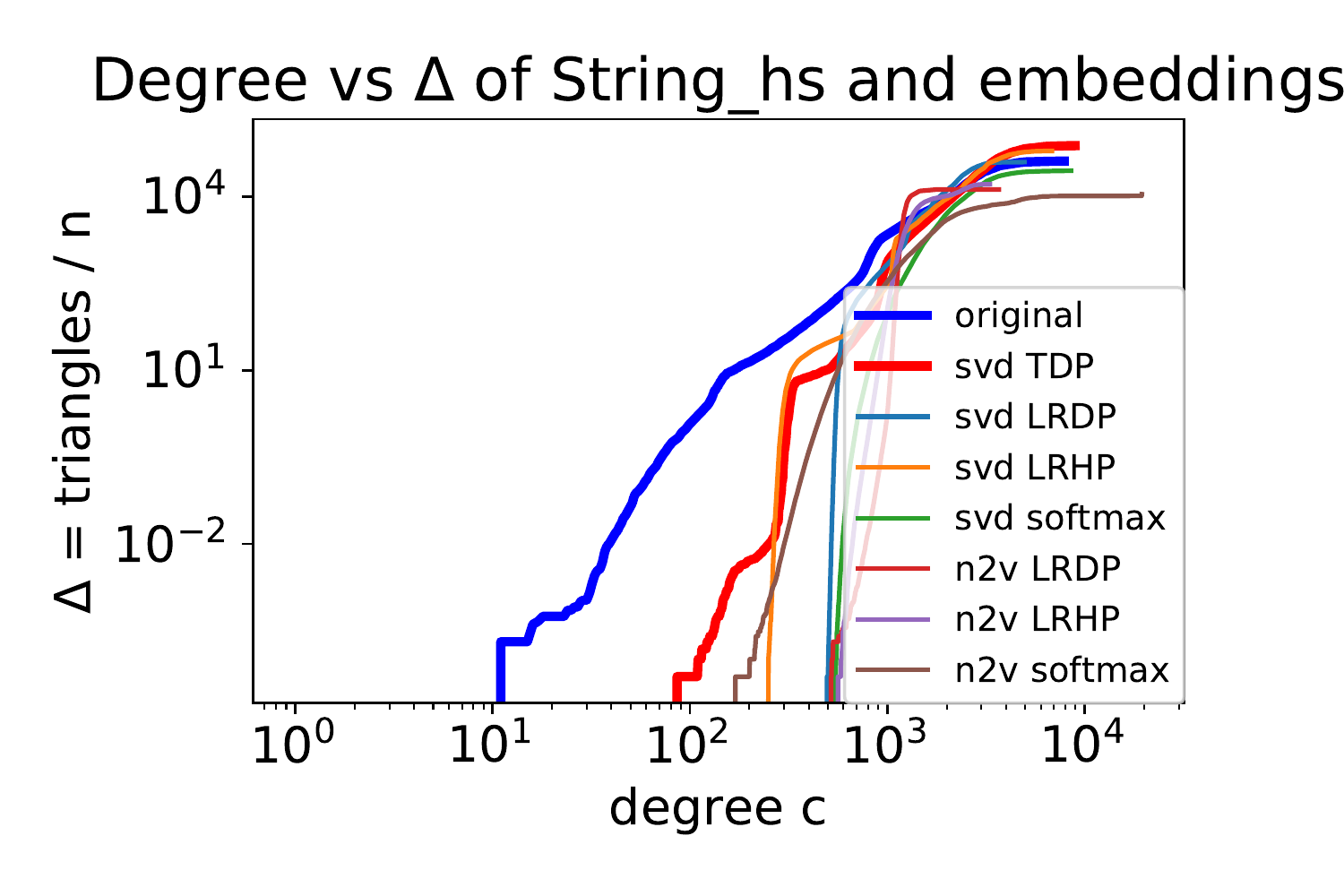} 
        \includegraphics[scale=.37]{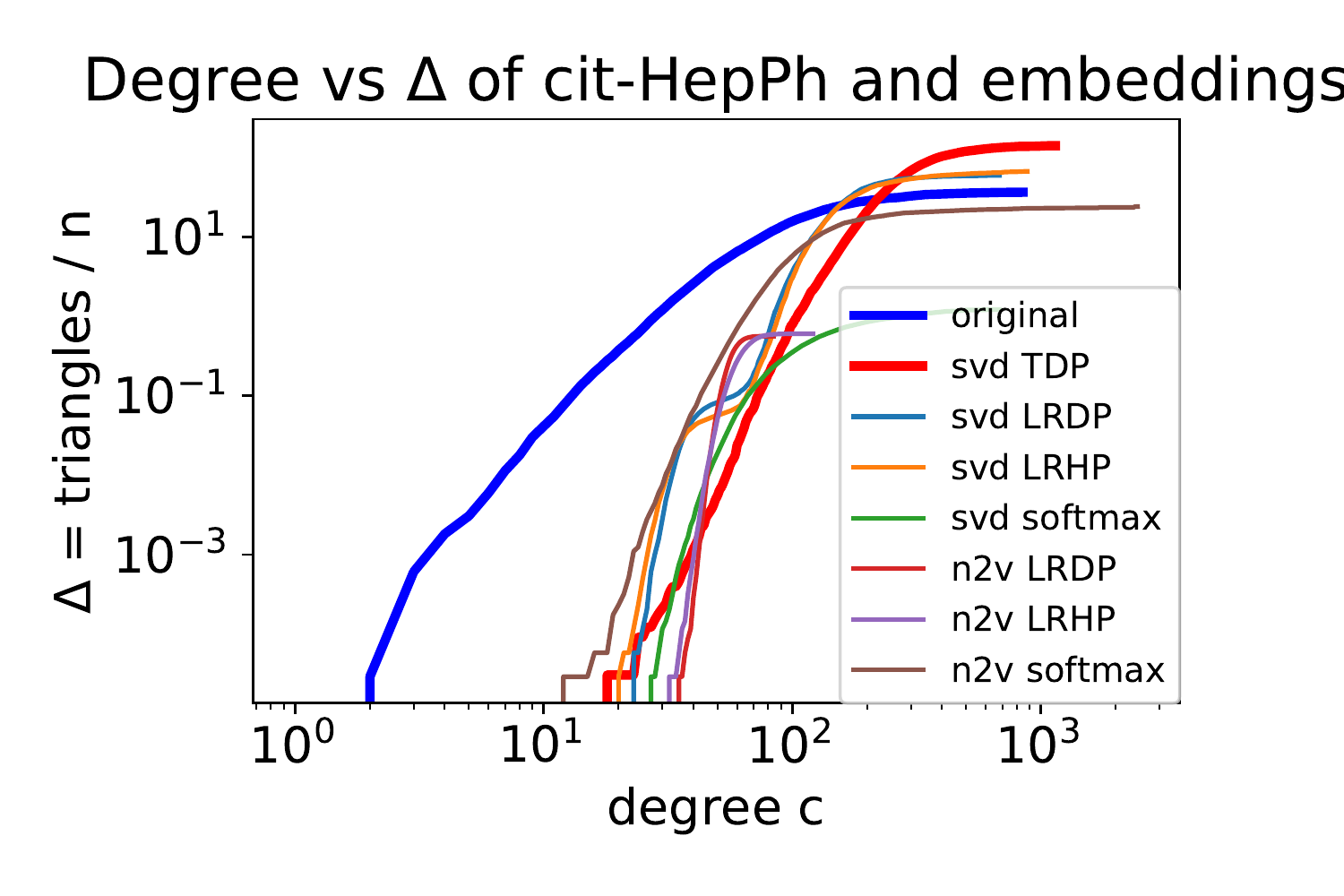}
				\captionof{figure}{\small Plots of degree $c$ vs $\Delta$: For each network,
        we plot $c$ versus the total number of triangles only involving vertices of
        degree at most $c$. We divide the latter by the number of vertices,
        so it corresponds to $\Delta$, as in the main definition.
        In each subfigure, we plot these
        both for the original graph, and the maximum $\Delta$ in 
        a set of 100 samples from a 100-dimensional embedding. Observe how the embeddings generate
        graphs with very few triangles among low degree vertices. The gap in $\Delta$ for low degree is
        2-3 orders of magnitude in all instances.} \label{fig:triangle-distros}
\end{figure*}

\begin{figure*}[h!]
        \includegraphics[scale=.3]{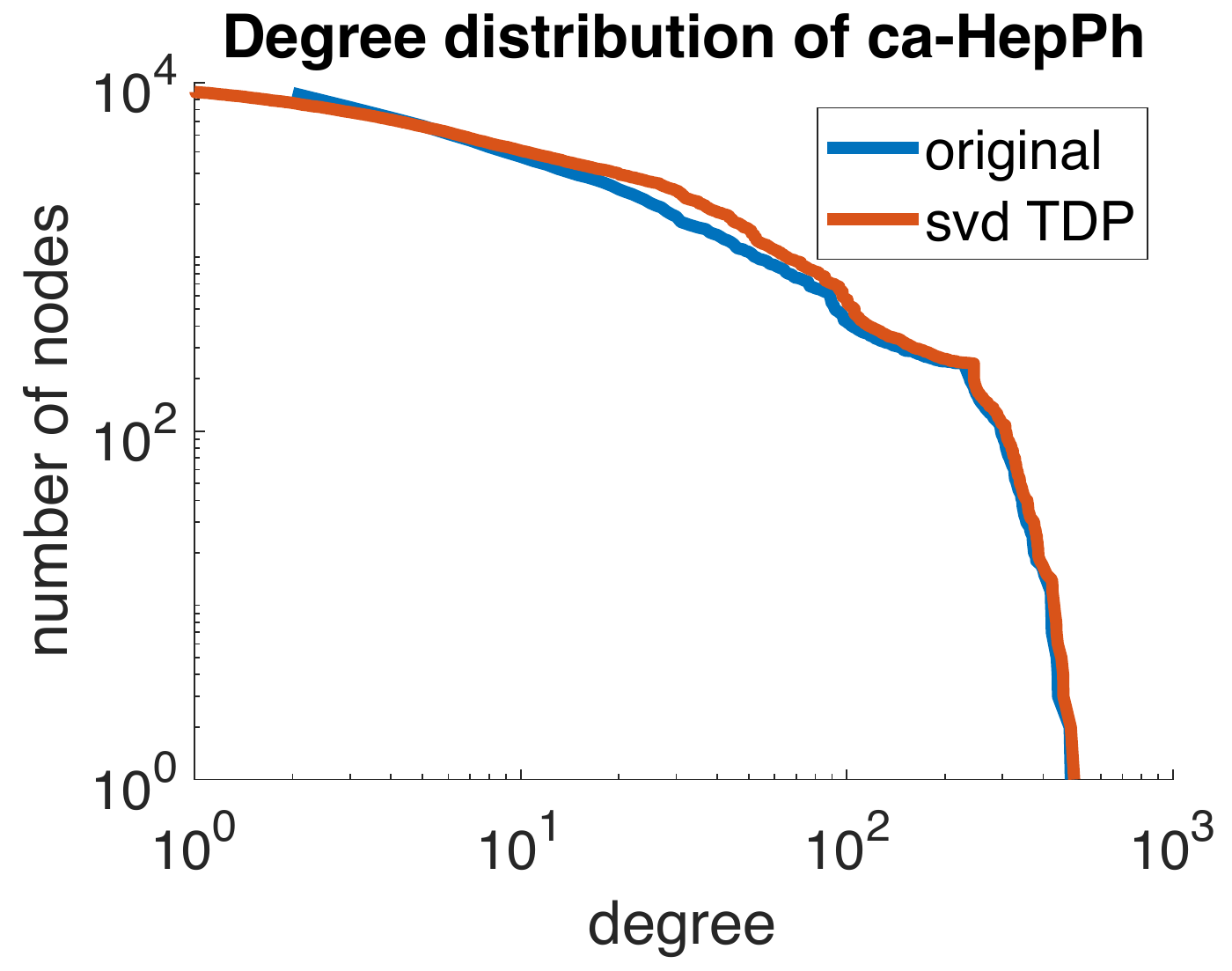}
        \includegraphics[scale=.3]{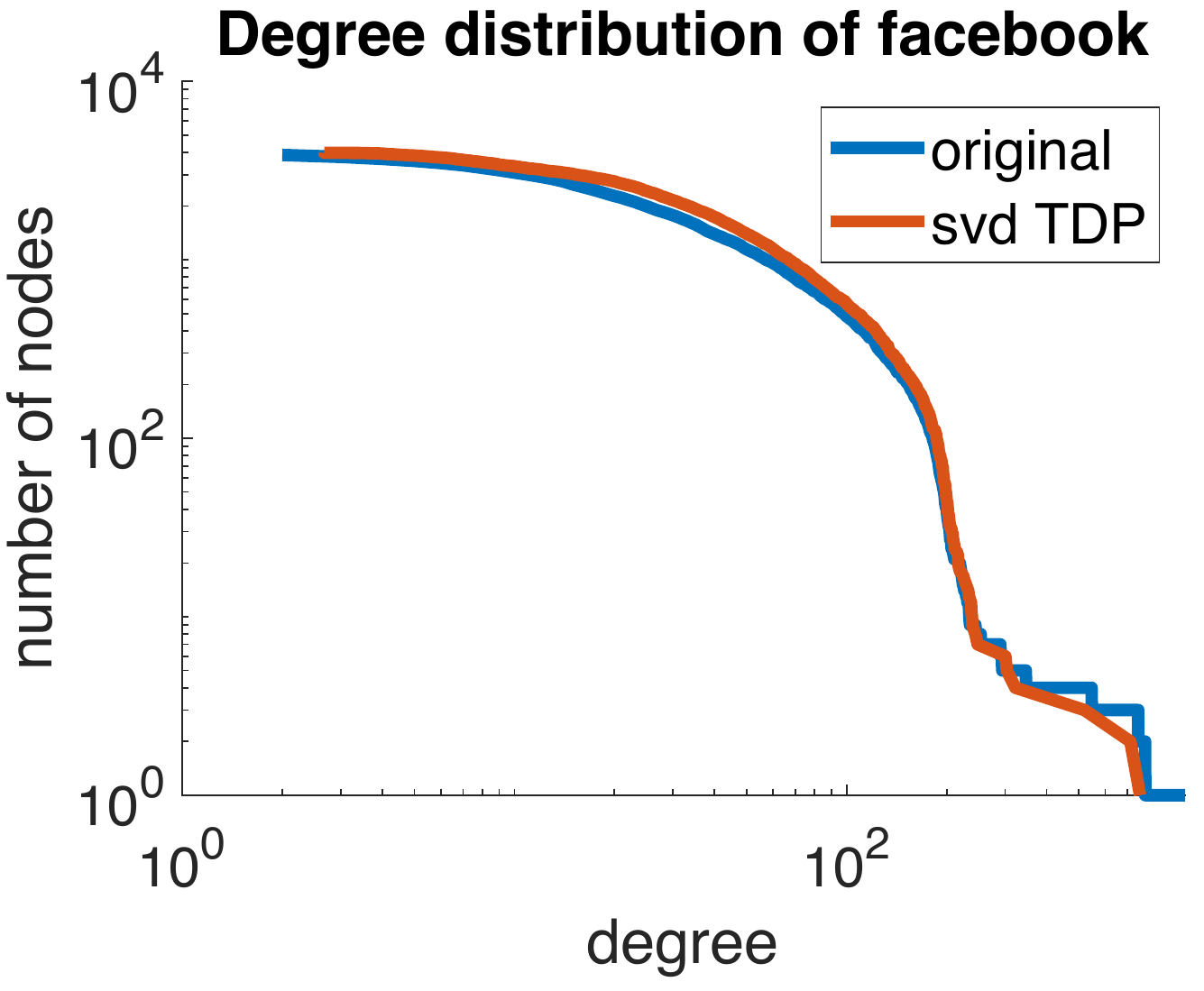}
        \includegraphics[scale=.3]{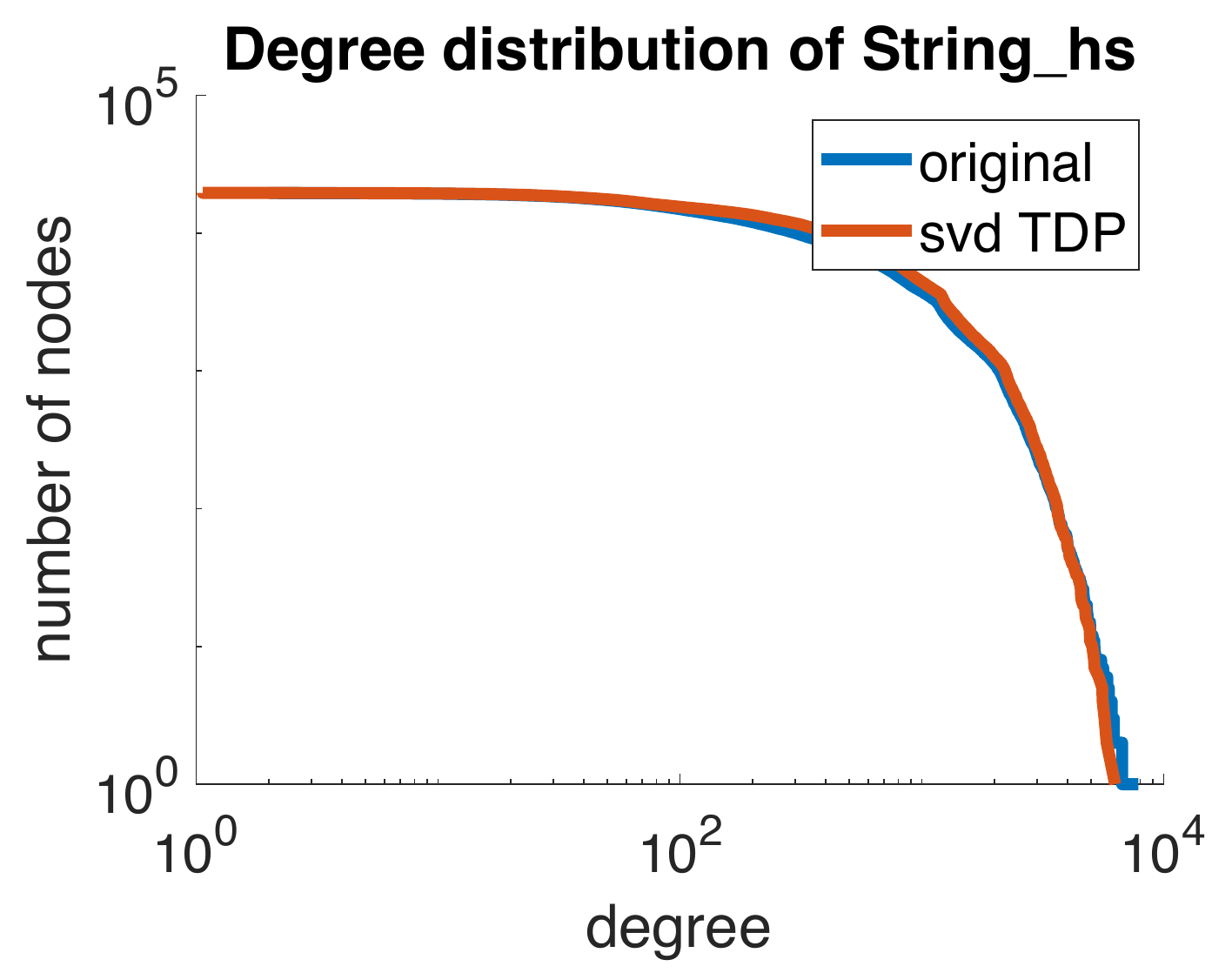} 
        \includegraphics[scale=.3]{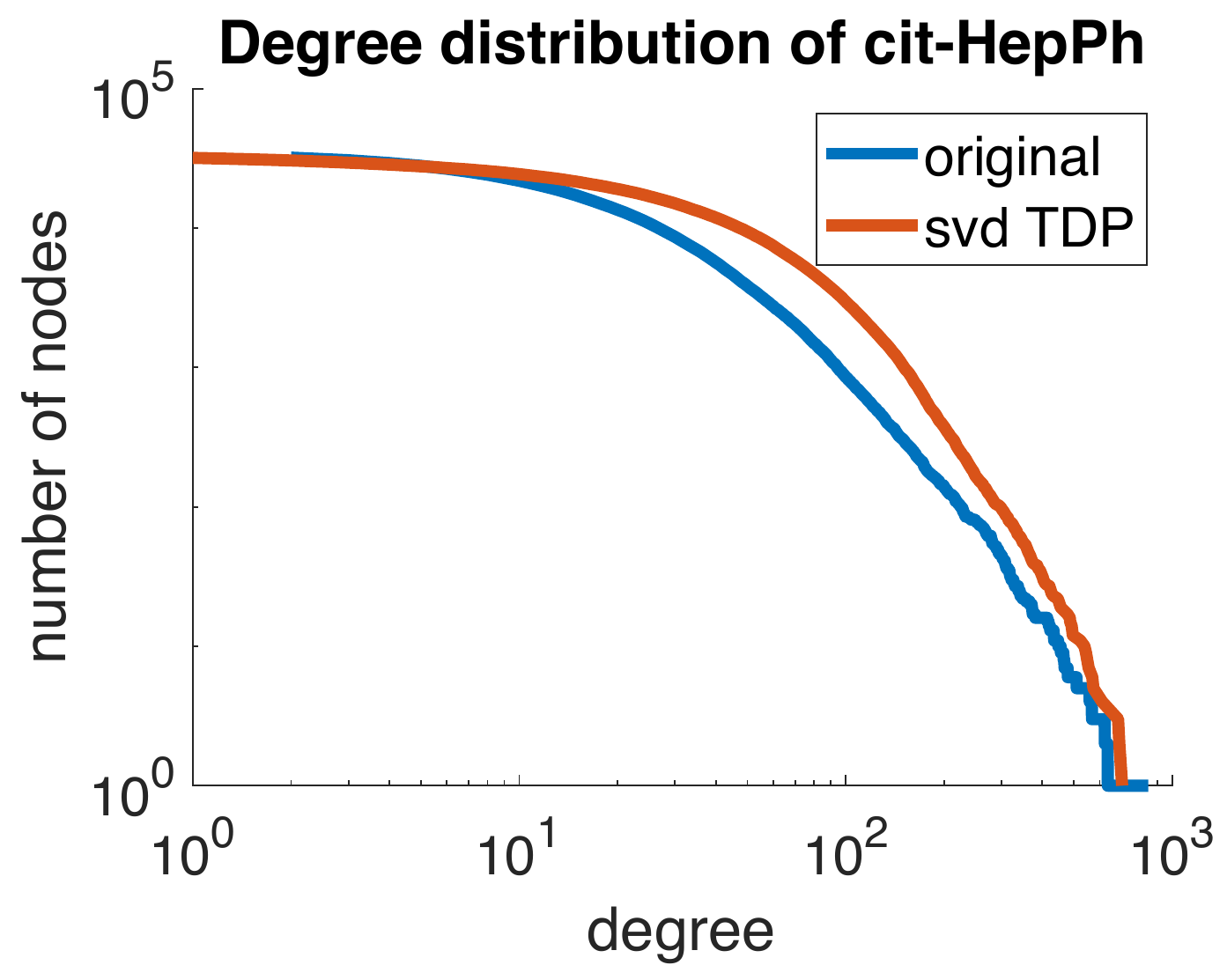}
				\captionof{figure}{\small Plots of degree distributions: For each network,
                we plot the true degree distribution vs the expected degree distribution of a 100-dimensional
                embedding. Observe how the embedding does capture the degree distribution quite accurately 
                at all scales.} \label{fig:degdist}
\end{figure*}

\subsection{Triangle distributions}
  To generate Figure~\ref{fig:intro-triangle-distro} and Figure~\ref{fig:triangle-distros}, we calculated the number of triangles incident to vertices of different degrees in both the original graphs and the graphs generated from the embeddings. Each of the plots shows the number of triangles in the graph on the vertical axis and the degrees of vertices on the horizontal axis. Each curve corresponds to some graph, and each point $(x, y)$ in a given curve shows that the graph contains $y$ triangles if we remove all vertices with degree at least $x$. We then generate
  100 random samples from the 100-dimensional embedding, as given by SVD (described above). For each value of $c$,
  we plot the maximum value of $\Delta$ over all the samples. This is to ensure that our results are not affected
  by statistical variation (which was quite minimal).

\subsection{Alternate graph models} \label{sec:alt}

{ We consider three other functions of the dot product, to construct graph distributions
from the vector embeddings. \new{Details on
parameter settings and the procedure used for the optimization are given in the SI.}

{\em Logistic Regression on the Dot Product (LRDP):} We consider the probability
    of an edge $(i,j)$ to be the logistic function $L(1+\exp(-k(\vec{v}_i \cdot \vec{v}_j - x_0)))^{-1}$,
    where $L, k, x_0$ are parameters. Observe that the range of this function is $[0,1]$,
    and hence can be interpreted as a probability. We tune these parameters to fit the expected
    number of edges, to the true number of edges. Then, we proceed as in the TDP experiment.
    We note that the TDP can be approximated by a logistic function, and thus the LRDP
    embedding is a ``closer fit" to the graph than the TDP embedding.

{\em Logistic Regression on the Hadamard Product (LRHP):} This is inspired
    by linear models used on low-dimensional embeddings~\cite{HaYiLe17}. Define the Hadamard
    product $\vec{v}_i \odot \vec{v}_j$ to be the $d$-dimensional vector where the $r$th
    coordinate is the product of $r$th coordinates. We now fit a logistic function over
    linear functions of (the coordinates of) $\vec{v}_i \odot \vec{v}_j$. This is a significantly
    richer model than the previous model, which uses a fixed linear function (sum). Again,
    we tune parameters to match the number of edges.
    The fitting of LRDP and LRHP was done
    using the Matlab function {\tt glmfit} (Generalized Linear Model Regression Fit)
    \cite{matlab}. The distribution parameter was set to ``binomial", since
    the total number of edges is distributed as a weighted binomial.

{\em Softmax:} This is inspired by low-dimensional embeddings for random walk
    matrices~\cite{PeAlSk14,grover2016node2vec}. The idea is to make the probability
    of edge $(i,j)$ to be proportional to softmax, 
    $\exp(\vec{v}_i \cdot \vec{v}_j)/\sum_{k \in [n]} \vec{v}_i \cdot \vec{v}_k$.
    This tends to push edge formation even for slightly higher dot products, and one might
    imagine this helps triangle formation. We set the proportionality constant separately for
    each vertex to ensure that the expected degree is the true degree. The probability matrix
    is technically undirected, but we symmetrize the matrix. 

{\bf {\sc node2vec} experiments}: 
\new{We also applied {\sc node2vec},
a recent deep learning based graph embedding method~\cite{grover2016node2vec}, to generate vector representations
of the vertices. 
We used the optimized C++
implementation~\cite{n2vcpp} for {\sc node2vec}, which is equivalent
to the original implementation provided by the authors~\cite{n2v}. For all our experiments, we use the default
settings of walk length of 80, 10 walks per node, p=1 and q=1. 
The {\sc node2vec} algorithm tries to model the random walk
matrix associated with a graph, not the raw adjacency matrix. The dot products
between the output vectors $\vec{v}_i \cdot \vec{v}_j$ are used to model 
the random walk probability of going from $i$ to $j$, rather than the presence of an edge. It does not make
sense to apply the TDP function on these dot products, since this will
generate (in expectation) only $n$ edges (one for each vertex). We apply
the LRDP or LRHP functions, which use the {\sc node2vec} vectors as inputs to a machine learning
model that predicts edges.
}

In Figures~\ref{fig:intro-triangle-distro} and~\ref{fig:triangle-distros}, we show results for all the datasets. 
We note that for all datasets and all embeddings, the models fail to capture the low-degree triangle behavior.


\subsection{Degree distributions}

{We observe that the low-dimensional embeddings obtained from SVD and the truncated dot product
can capture the degree distribution accurately. In Figure~\ref{fig:degdist}, we plot the degree
distribution (in loglog scale) of the original graph with the expected degree distribution
of the embedding. For each vertex $i$, we can compute its expected degree by the sum $\sum_i p_{ij}$,
where $p_{ij}$ is the probability of the edge $(i,j)$. In all cases, the expected degree distributions
is close to the true degree distributions, even for lower degree vertices. 
The embedding successfully captures the ``first order"
connections (degrees), but not the higher order connections (triangles). We believe that this
reinforces the need to look at the triangle structure to discover the weaknesses of low-dimensional
embeddings.}

%
%
%
\subsection{Detailed relationship between rank and triangle structure} \label{sec:relation}

For the smallest {\tt Facebook} graph, we were able to compute the entire set of eigenvalues.
This allows us to determine how large a rank is required to recreate the low-degree triangle structure.
In Figure~\ref{fig:full-spectrum}, for varying rank of the embedding,
we plot the corresponding triangle distribution.
In this plot, we choose the embedding given by the eigendecomposition
(rather than SVD), since it is guaranteed to converge to the correct triangle distribution
for an $n$-dimensional embedding ($n$ is the number of vertices). The SVD and eigendecomposition
are  mostly identical for large singular/eigenvalues, but tend to be different (up to a sign)
for negative eigenvalues.

We observe that even a 1000 dimensional embedding does not capture the $c$ vs $\Delta$ plots
for low degree. Even the rank 2000 embedding
is off the true values, though it is correct to within an order of magnitude.
This is strong corroboration of our main theorem, which says that near linear
rank is needed to match the low-degree triangle structure. 

\begin{figure}
        \includegraphics[scale=0.5]{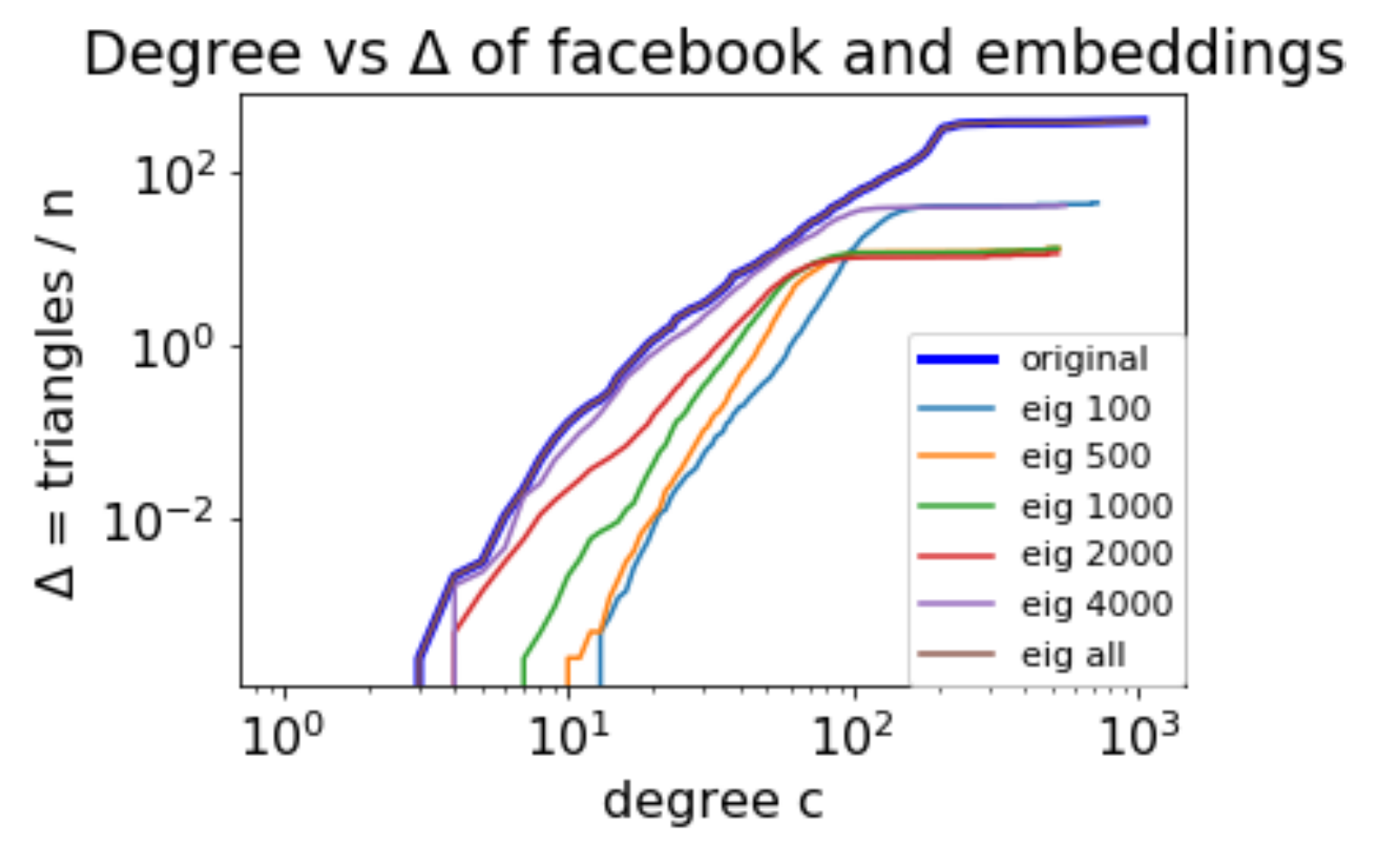}
				\captionof{figure}{\small Plots of degree $c$ vs $\Delta$, for varying rank: 
                For the {\tt Facebook} social network, for varying rank of embedding, we plot $c$ versus the total number of triangles only involving vertices of
        degree at most $c$. The embedding is generated by taking the top eigenvectors. Observe how
        even a rank of 2000 does not suffice to match the true triangle values for low degree.} \label{fig:full-spectrum}
\end{figure}

\bibliographystyle{alpha}
\bibliography{../embeddings}
\end{document}